\documentclass[letterpaper]{article} 
\usepackage{aaai2026}  
\usepackage{times}  
\usepackage{helvet}  
\usepackage{courier}  
\usepackage[hyphens]{url}  
\usepackage{graphicx} 
\usepackage{natbib}  
\usepackage{caption} 
\usepackage{algorithm}
\usepackage{algorithmic}

\usepackage{multirow}
\usepackage{booktabs}
\usepackage{amsmath}
\usepackage{amssymb}
\usepackage{amsthm}
\usepackage[table]{xcolor}
\usepackage{makecell}
\newtheorem{theorem}{Theorem}

\theoremstyle{definition}
\newtheorem{definition}{Definition}

\theoremstyle{remark}

\usepackage{newfloat}
\usepackage{listings}
\DeclareCaptionStyle{ruled}{labelfont=normalfont,labelsep=colon,strut=off} 
\floatstyle{ruled}
\newfloat{listing}{tb}{lst}{}
\floatname{listing}{Listing}
\title{SAGE: Spuriousness-Aware Guided Prompt Exploration for Mitigating Multimodal Bias}
\author{
    Wenqian Ye\textsuperscript{\rm 1}, Di Wang\textsuperscript{\rm 1}, Guangtao Zheng\textsuperscript{\rm 2}, Bohan Liu\textsuperscript{\rm 1}, Aidong Zhang\textsuperscript{\rm 1}\thanks{Corresponding Author.}
}
\affiliations{
    \textsuperscript{\rm 1}Department of Computer Science, University of Virginia, USA\\
    \textsuperscript{\rm 2}Accenture, USA\\

    \{wenqian, azm7tq, qzp4ta, aidong\}@virginia.edu, zhguangt@gmail.com
}

\usepackage{bibentry}

\begin{document}

\maketitle

\begin{abstract}
Large vision-language models, such as CLIP, have shown strong zero-shot classification performance by aligning images and text in a shared embedding space. However, CLIP models often develop multimodal spurious biases, which is the undesirable tendency to rely on spurious features. For example, CLIP may infer object types in images based on frequently co-occurring backgrounds rather than the object's core features. This bias significantly impairs the robustness of pre-trained CLIP models on out-of-distribution data, where such cross-modal associations no longer hold. Existing methods for mitigating multimodal spurious bias typically require fine-tuning on downstream data or prior knowledge of the bias, which undermines the out-of-the-box usability of CLIP. In this paper, we first theoretically analyze the impact of multimodal spurious bias in zero-shot classification. Based on this insight, we propose Spuriousness-Aware Guided Exploration (SAGE), a simple and effective method that mitigates spurious bias through guided prompt selection. SAGE requires no training, fine-tuning, or external annotations. It explores a space of prompt templates and selects the prompts that induce the largest semantic separation between classes, thereby improving worst-group robustness.
Extensive experiments on four real-world benchmark datasets and five popular backbone models demonstrate that SAGE consistently improves zero-shot performance and generalization, outperforming previous zero-shot approaches without any external knowledge or model updates.
\end{abstract}

\begin{links}
    \link{Code}{https://github.com/wenqian-ye/spurious_vlm}
\end{links}

\section{Introduction}
\label{sec:intro}
Pre-trained models hold promising potential for open-set classification without the need for additional data collection or training. Pre-trained vision-language models (VLMs) \cite{radford2021learning,jia2021scaling,li2021align,wangsimvlm,li2022grounded}, such as contrastive language-image pre-training (CLIP) models \cite{radford2021learning}, have demonstrated a strong zero-shot prediction capability across diverse downstream tasks. They typically consist of a pre-trained image encoder and a text encoder, from which vision and text representations are aligned in a shared joint embedding space. Thus, zero-shot classification of an image can be achieved simply by matching the image representation to a set of candidate text representations. 



\begin{figure}[t]
    \centering
    \includegraphics[width=\linewidth]{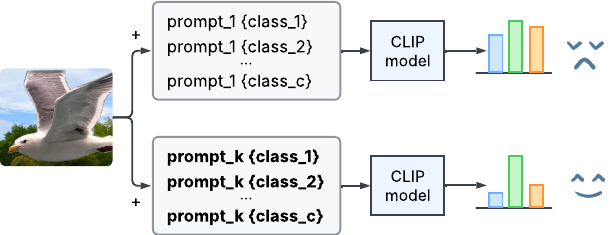}
    \caption{Prompts with greater separation between class similarity scores (e.g., prompt\_k) yield robust zero-shot performance under spurious correlations, whereas those with smaller score differences (e.g., prompt\_1) tend to yield poorer discrimination and worst-group performance.}
    \label{fig:motivation}
\end{figure}


However, recent studies \cite{you2024calibrating,adila2024zeroshot,dehdashtianfairerclip} have found that pre-trained CLIP models often develop undesirable tendencies to rely on spurious correlations between non-essential features and target labels across modalities when making predictions. For example, the class label \texttt{landbird} may become spuriously associated with \texttt{land background} due to frequent co-occurrence in the pre-training data \cite{zheng2024spuriousness,zheng2024learning}. As a result, a CLIP model may incorrectly predict a \texttt{waterbird} as a \texttt{landbird} simply because it appears in a \texttt{land background}. This kind of biased prediction behavior, referred to as multimodal spurious bias, severely impairs the zero-shot generalization ability of CLIP models on out-of-distribution data where such spurious correlations no longer hold \cite{ye2024mm}. For instance, the correlation between \texttt{landbird} and \texttt{land background} may not exist in the downstream out-of-distribution evaluation setting.

Mitigating multimodal spurious bias is essential for ensuring robust generalization across downstream tasks. Existing methods vary significantly in their approaches. Some \cite{yang2023mitigating,you2024calibrating,zhang2024amend,dehdashtianfairerclip} adopt fine-tuning strategies, focusing on task-specific biases and requiring additional data. Although these methods improve robustness to multimodal spurious bias over the vanilla zero-shot approach, they rely on labeled data and do not address the zero-shot setting. ROBOSHOT \cite{adila2024zeroshot} mitigates spurious bias within the language modality without training data, but typically requires specifying spurious attributes by prompting a large language model (LLM) for each downstream task. TIE* \cite{ICLR2025_2649b1d4}, though not relying on LLMs, directly uses spurious attributes to obtain pseudo spurious labels for multimodal bias mitigation.

To mitigate multimodal spurious bias without relying on prior knowledge or external models, we propose a training-free framework, namely \textbf{S}puriousness-\textbf{A}ware \textbf{G}uided \textbf{E}xploration (SAGE). We begin by formally defining multimodal spurious bias and analyzing its impact on zero-shot classification. Our insight is that prompts with larger differences in inter-class similarity tend to better capture core class semantics, which helps reduce reliance on spurious correlations. As shown in Figure~\ref{fig:motivation}, the top part illustrates prompt\_1, where the green bar is only slightly higher than the blue bar, reflecting weak class separation and lower zero-shot performance. The bottom part shows prompt\_k, where the green bar is higher than the lowest blue bar, indicating stronger class discrimination and improved predictive accuracy. Our theoretical and empirical results suggest that higher \textit{separation scores} are associated with greater focus on essential class features rather than spurious ones, thereby improving zero-shot robustness.


Based on this insight, SAGE utilizes a set of diverse candidate prompt templates commonly used with CLIP models or their variants. For each image, SAGE calculates the similarity scores between the image and the class labels under different prompt templates. The prompt template with top greatest difference between the highest and lowest class scores is selected for zero-shot inference. SAGE works entirely without fine-tuning or external supervision and can be applied to any zero-shot vision-language model. Extensive experiments on four benchmark datasets and five backbone models demonstrate that SAGE consistently enhances zero-shot accuracy while effectively mitigating multimodal spurious bias.




\section{Related Work}
\paragraph{Spurious bias in single data modality.} Spurious bias refers to the reliance of models on spurious correlations between input features and targets, leading to poor generalization on out-of-distribution data~\cite{beery2018recognition,geirhos2020shortcut,ijcai2025p795,ye2025cleverhansmiragecomprehensive}. Existing methods typically mitigate spurious bias by retraining models with labeled data, where spurious correlations are either explicitly annotated via group labels \cite{sagawa2019distributionally,kirichenko2022last,deng2024robust}, implicitly identified through group inference \cite{nam2022spread,ye2025improving,zhengneurontune}, or sample reweighting \cite{nam2020learning,liu2021just,qiu2023simple,labonte2024towards}. Our work addresses the relatively under-explored challenge of mitigating spurious bias in a zero-shot multimodal setting where no retraining data is available.


\paragraph{Debiasing fine-tuned VLMs.}  Multimodal spurious bias refers to the tendency of models to rely on spurious correlations in one modality (e.g., image background) to infer targets in another (e.g., object names). In VLMs, such bias may arise from misalignment between modalities during pre-training or fine-tuning~\cite{tong2024eyes,sun2024aligning}. Existing approaches mitigate this bias in fine-tuned VLMs using contrastive learning with or without group labels~\cite{yang2023mitigating,zhang2022contrastive,you2024calibrating}, or by disentangling spurious and core features via prompt tuning~\cite{zhang2024amend} or latent projection~\cite{dehdashtianfairerclip}. In contrast, our method targets the zero-shot setting without any downstream data, which is applicable on a broader real-world scenarios.


\paragraph{Zero-shot debiasing.}  Debiasing in the zero-shot setting aims to mitigate multimodal spurious bias learned during pre-training, where target texts may align with spurious image features~\cite{ge2023improving}. Existing methods typically leverage text data: \citet{chuang2023debiasing} use prompts with known spurious attributes to adjust CLIP classifier weights, while \citet{adila2024zeroshot} extract core and spurious attributes via LLMs to enhance image features. Moreover, \citet{ICLR2025_2649b1d4} rely on explicit spurious attribute information to generate pseudo-labels that help reduce bias in multimodal embeddings.  In contrast, our method mitigates bias by selecting prompts based on a \textit{separation score} computed from similarity differences between class labels, without relying on LLMs or prior knowledge.

\begin{figure*}[t]
    \centering
    \includegraphics[width=\linewidth]{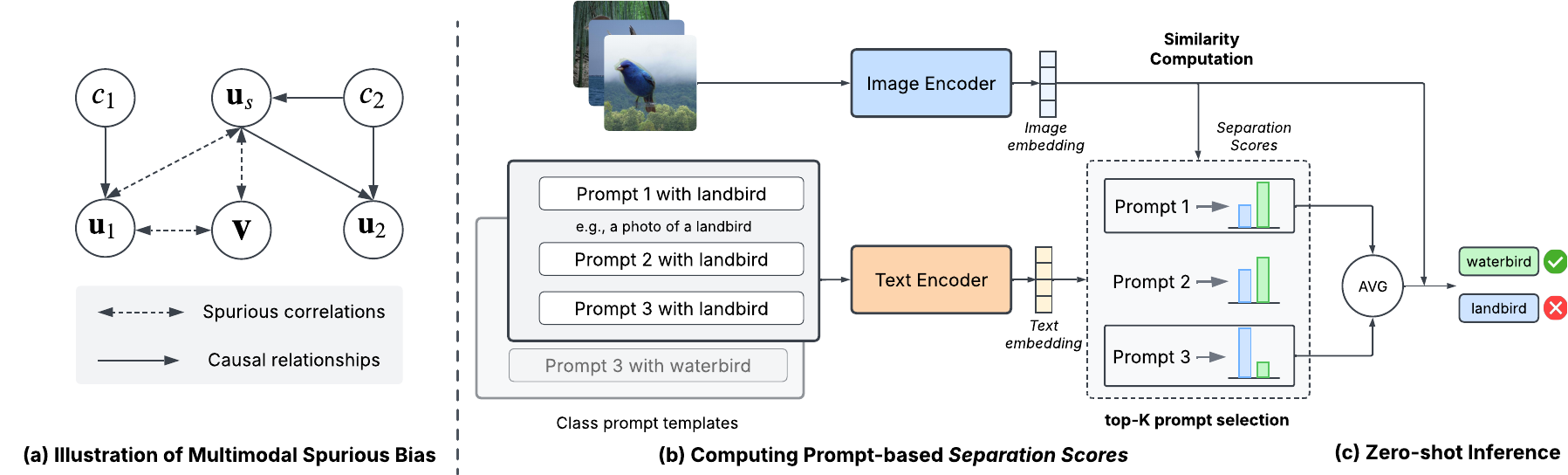}
    \caption{Method overview. (a) Illustration of multimodal spurious bias, where $c_2$ denotes a class label, $\mathbf{v}$ denotes an image representation, $\mathbf{u}_s$ denotes a textual spurious feature, $\mathbf{u}_1$ and $\mathbf{u}_2$ denote text representations for the class $c_1$ and $c_2$ respectively. (b) For each test image, we evaluate $M$ prompt templates and compute a \textit{separation score} that measures how well each prompt distinguishes between classes in the joint image-text space. The top-$K$ templates with the highest scores are selected. (c) Zero-shot classification is then performed by ensembling predictions from the $K$ class-discriminative prompts selected for that image.}
    \label{fig:method-overview}
\end{figure*}

\section{Methodology}
 We first theoretically analyze the multimodal spurious bias in VLMs. Based on the insights gained in the analysis, we introduce spuriousness-aware guided prompt exploration (SAGE) that selects prompts based on a \textit{separation score} to mitigate such bias in a zero-shot setting.
\subsection{Preliminary}

A CLIP \cite{radford2021learning} model is trained to align the representation of an image $x$ from its vision encoder $\phi$ and the representation of a text description $t$ from its text encoder $\psi$ in a joint embedding space when the text description $t$ matches with the image $x$. Specifically,  let $\mathbf{v}=\phi(x)\in\mathbb{R}^{D}$ denote the vision representation for the image $x$ and $\mathbf{u}=\psi(t)\in\mathbb{R}^{D}$ be the text representation for the text description $t$, where $D$ is the number of embedding dimensions. Then, the CLIP training objective \cite{radford2021learning} essentially aims to maximize the probability of $\mathbf{v}$ given $\mathbf{u}$ and the probability of $\mathbf{u}$ given $\mathbf{v}$ over all training image-text pairs, i.e.,
\begin{equation}\label{eq:ori-obj}
\phi,\psi=\arg\max_{\phi',\psi'}\mathbb{E}_{p(x,t)}\Big(p(\mathbf{v}|\mathbf{u}) + p(\mathbf{u}|\mathbf{v})\Big),
\end{equation}
where  $p(x,t)$ denotes the joint distribution of matching image-text pairs in the training set. During training, CLIP models learn to align the embeddings of matching image-text pairs—for example, bringing together the representation of a landbird image and the phrase ``a photo of a landbird". At the same time, it pushes apart the representations of mismatched pairs, such as a waterbird image with the same phrase. Ideally, for a matching image-text pair $(x,t)$, we will obtain $p(\mathbf{v}|\mathbf{u}) \approx p(\mathbf{u}|\mathbf{v})$ after training.


\paragraph{Zero-shot classification.} Given an image $x$ belonging to one of $C$ classes $\{c_i\}_{i=1}^C$, zero-shot classification first constructs $C$ text descriptions by  inserting each class name $c_i$ into a predefined text template, such as ``a photo of a [$c_i$]" (e.g., ``a photo of a landbird"). Each description is then encoded into a text representation $\mathbf{u}_i$ for each class $c_i$. Then, the zero-shot prediction $\hat{y}$ is:
\begin{equation}\label{eq:zero-shot-classification}
    \hat{y}=\arg\max_{i} p(\mathbf{u}_i|\mathbf{v})=\arg\max_{i} \frac{\mathbf{v}^T\mathbf{u}_i}{\|\mathbf{v}\|_2\|\mathbf{u}_i\|_2},
\end{equation}
where $\mathbf{v}$ is the vision representation for the input image $x$, $\|\cdot\|_2$ is the Euclidean norm of a vector, and $p(\mathbf{u}_i|\mathbf{v})$ is defined to be proportional to $\mathbf{v}^T\mathbf{u}_i$.

\subsection{Multimodal Spurious Bias}
In practice, a given text description $t$ may not fully describe the content in $x$. For example, $x$ could be an image depicting a landbird with a land background, and $t$ could simply be ``a photo of a landbird", which only describes the primary object in the image. When a CLIP model learns to align many such image-text pairs where land backgrounds spuriously correlate with the target ``landbird", then the model may inadvertently learn to align the representation of ``a photo of a landbird" with the representation of land backgrounds, instead of the defining features of landbirds. The misalignment causes a \textit{multimodal spurious bias} in the model which tends to use land backgrounds in images to infer their descriptions. Due to the misalignment, an image of waterbird with a land background is incorrectly paired with the description ``a photo of a landbird".


To formally define multimodal spurious bias,  we introduce $\mathbf{u}_s\in\mathbb{R}^D$ to represent a latent textual spurious feature, such as the missing ``land background" in the description ``a photo of a landbird". With $\mathbf{u}_s$, we can conveniently expand $p(\mathbf{v}|\mathbf{u})$ and $p(\mathbf{u}|\mathbf{v})$ in  \eqref{eq:ori-obj} as the marginalization over all possible textual spurious features, i.e.,
\begin{equation}\label{eq:image-conditioned-on-text}
p(\mathbf{v}|\mathbf{u}) = \int_{\mathbf{u}_s}p(\mathbf{v}|\mathbf{u},\mathbf{u}_s) p(\mathbf{u}_s|\mathbf{u})d\mathbf{u}_s,
\end{equation}
and
\begin{equation}\label{eq:text-conditioned-on-image}
p(\mathbf{u}|\mathbf{v})=\int_{\mathbf{u}_s}p(\mathbf{u}|\mathbf{v},\mathbf{u}_s)p(\mathbf{u}_s|\mathbf{v})d\mathbf{u}_s.
\end{equation}
In the pre-training data, if the majority of images with their text representation $\mathbf{u}$ have a spurious feature represented by $\mathbf{u}_s$, then a CLIP model may learn the strong correlations between the spurious feature $\mathbf{u}_s$ and the image representation $\mathbf{v}$ as well as the text representation $\mathbf{u}$. As a result, the model will develop multimodal spurious bias and we will have $p(\mathbf{u}_s|\mathbf{u})\approx 1$ and $p(\mathbf{u}_s|\mathbf{v})\approx 1$. We formally define multimodal spurious bias as follows.

\begin{definition}[\textbf{Multimodal spurious bias}]\label{def:multimodal-spurious-bias}
    Consider a pre-trained CLIP model consisting of a vision encoder $\phi$ and a text encoder $\psi$. Given an image-text pair $(x,t)$ and a latent spurious feature $\mathbf{u}_s$,  a multimodal spurious bias in the model relevant to $\mathbf{u}_s$ satisfies the following conditions:
    \begin{equation}\label{eq:image-conditioned-on-text-approx}
p(\mathbf{v}|\mathbf{u}) \approx p(\mathbf{v}|\mathbf{u},\mathbf{u}_s),
\end{equation}
and
\begin{equation}\label{eq:text-conditioned-on-image-approx}
p(\mathbf{u}|\mathbf{v})\approx p(\mathbf{u}|\mathbf{v},\mathbf{u}_s),
\end{equation}
where $\mathbf{v}=\phi(x)$ and $\mathbf{u}=\psi(t)$.
\end{definition}

The above conditions indicate that $p(\mathbf{u}_s|\mathbf{u})\approx 1$ and $p(\mathbf{u}_s|\mathbf{v})\approx 1$ are based on Eq.~\eqref{eq:image-conditioned-on-text} and Eq.~\eqref{eq:text-conditioned-on-image}, and
the pre-trained model tends to align $\mathbf{v}$ and $\mathbf{u}$ with $\mathbf{u}_s$. This indicates a misalignment between the vision representation $\mathbf{v}$ and the text representation $\mathbf{u}$.  When the pre-trained model is tested on the data with $p(\mathbf{u}_s|\mathbf{u})\ll 1$ and $p(\mathbf{u}_s|\mathbf{v})\ll 1$, i.e., the spurious features in the test data no longer have strong correlations with input images and the corresponding text descriptions compared to the training data, such as the waterbird image with a land background where a land background is no longer associated with landbird, then the model may struggle on most of the test data, showing degraded zero-shot classification performance.

\subsection{Theoretical Insights}\label{sec:theoretical-insights}

We first theoretically analyze how multimodal spurious bias affects zero-shot classification. The insights derived from our analysis will guide the design of our multimodal spurious bias mitigation method in the following section.

Without loss of generality, we consider a zero-shot classification task with two classes $c_1$ and $c_2$. Given a prompt template, we can obtain text representations for the two classes as $\mathbf{u}_1$ and $\mathbf{u}_2$. Consider an image representation $\mathbf{v}$ from class $c_2$ with an unknown spurious feature described by the text representation $\mathbf{u}_s$. The zero-shot prediction $\hat{y}$ can be obtained as follows,
\begin{equation}
    \hat{y}=\arg\max_{i\in\{1,2\}}p(\mathbf{u}_i|\mathbf{v}).
\end{equation}
We assume a multimodal spurious bias between $\mathbf{u}_1$,  $\mathbf{v}$, and $\mathbf{u}_s$, as indicated by the dashed arrows in Figure \ref{fig:method-overview}(a). Then, the zero-shot prediction may be biased towards the class label $c_1$, instead of the true class label $c_2$, as supported by the following theorem.

\begin{theorem}
    Consider a pre-trained CLIP model from which we obtain two text representations $\mathbf{u}_1$, $\mathbf{u}_2$ for the classes $c_1$ and $c_2$ respectively, an image representation $\mathbf{v}$ with the class label $c_2$, and a textual spurious feature $\mathbf{u}_s$ related to $\mathbf{v}$. Assume $\mathbf{u}_1$, $\mathbf{v}$, and $\mathbf{u}_s$ formulate a multimodal spurious bias. Then, the model is biased towards predicting $\mathbf{v}$ as $c_1$ instead of its true class label $c_2$.
\end{theorem}

\begin{proof}
We first follow Eq. \eqref{eq:text-conditioned-on-image} to expand $p(\mathbf{u}_1|\mathbf{v})$, i.e.,
    \begin{align}
        p(\mathbf{u}_1|\mathbf{v}) &= \int_{\mathbf{u}_s}p(\mathbf{u}_1|\mathbf{v},\mathbf{u}_s)p(\mathbf{u}_s|\mathbf{v})d\mathbf{u}_s\label{eq:text-image-expansion}\\
        &\approx p(\mathbf{u}_1|\mathbf{v},\mathbf{u}_s)p(\mathbf{u}_s|\mathbf{v})\label{eq:text-image-expansion-approx}\\
        &= p(\mathbf{u}_s|\mathbf{u}_1)p(\mathbf{u}_1) \label{eq:text-image-expansion-approx-simplify},
    \end{align}
where the approximation in \eqref{eq:text-image-expansion-approx} uses the definition of multimodal spurious bias in Definition \ref{def:multimodal-spurious-bias}, and Eq. \eqref{eq:text-image-expansion-approx-simplify} can be derived via Bayes' theorem, i.e.,
\begin{align}
    p(\mathbf{u}_1|\mathbf{v},\mathbf{u}_s)=\frac{p(\mathbf{u}_1, \mathbf{u}_s|\mathbf{v})}{p(\mathbf{u}_s|\mathbf{v})}=\frac{p(\mathbf{u}_s|\mathbf{u}_1)p(\mathbf{u}_1)}{p(\mathbf{u}_s|\mathbf{v})},
\end{align}
where the last equality follows the fact that $p(\mathbf{u}_1,\mathbf{u}_s|\mathbf{v})=p(\mathbf{u}_1,\mathbf{u}_s)$, i.e., $\mathbf{u}_s$ and $\mathbf{u}_1$ do not depend on $\mathbf{v}$, as depicted in Figure \ref{fig:method-overview}(a). Therefore, we have the following inequality:
\begin{align}\label{eq:biased-prediction-analysis}
    \frac{p(\mathbf{u}_1|\mathbf{v})}{p(\mathbf{u}_2|\mathbf{v})}\approx \frac{p(\mathbf{u}_s|\mathbf{u}_1)p(\mathbf{u}_1)}{p(\mathbf{u}_2|\mathbf{v})} > 1,
\end{align}
where the inequality follows from the condition that $\mathbf{u}_1$, $\mathbf{v}$, and $\mathbf{u}_s$ formulate a multimodal spurious bias, i.e., $p(\mathbf{u}_s|\mathbf{u}_1)\approx 1$, $p(\mathbf{u}_2|\mathbf{v})\approx 0$  given that $p(\mathbf{u}_s|\mathbf{v})\approx 1$, and $p(\mathbf{u}_1)>0$ is a constant. Therefore, the model's prediction on $\mathbf{v}$ is biased towards the incorrect label $c_1$.
\end{proof}


Based on the above analysis, SAGE is motivated to choose such prompt template $\mathbf{u}$ that directly controls the multimodal spurious bias term $p(\mathbf{u}_s|\mathbf{u}_1)$. When a prompt induces spurious biases (i.e., $p(\mathbf{u}_s|\mathbf{u}_1) \approx 1$), the model's predictive probability on two classes $p(\mathbf{u}_1|\mathbf{v})$
and $p(\mathbf{u}_2|\mathbf{v})$ becomes arbitrarily close. This indicates a low class separation. Conversely, when another prompt template $\mathbf{u}'$ is less affected by spurious biases (i.e., $p(\mathbf{u}_s|\mathbf{u}_1') \ll 1$), the predictive probability for the incorrect class $p(\mathbf{u}_1'|\mathbf{v})$ decreases significantly, while the predictive probability for the correct class $p(\mathbf{u}_2'|\mathbf{v})$ remains high. This creates a large margin between $p(\mathbf{u}_2'|\mathbf{v})$ and $p(\mathbf{u}_1'|\mathbf{v})$. Therefore, maximizing class separation margin can serve as a robust and practical proxy that effectively mitigates spurious biases without prior knowledge of the spurious attributes in advance.


\begin{table*}[ht]
    \centering
    \footnotesize
    \setlength{\tabcolsep}{2pt}
    \resizebox{\textwidth}{!}{
    \begin{tabular}{llccc|ccc|ccc|ccc}
        \toprule
        \multirow{3}{*}{\centering Method} 
        & \multirow{3}{*}{\centering Model}
        & \multicolumn{3}{c}{\multirow{1}{*}{\centering Waterbirds}}
        & \multicolumn{3}{c}{\multirow{1}{*}{\centering CelebA}}
        & \multicolumn{3}{c}{\multirow{1}{*}{\centering PACS}}
        & \multicolumn{3}{c}{VLCS} 
        \\
        \cmidrule(lr){3-5}\cmidrule(lr){6-8}\cmidrule(lr){9-11}\cmidrule(lr){12-14}
        & 
        & AVG($\uparrow$) & WGA($\uparrow$) & HM($\uparrow$)
        & AVG($\uparrow$) & WGA($\uparrow$) & HM($\uparrow$)
        & AVG($\uparrow$) & WGA($\uparrow$) & HM($\uparrow$)
        & AVG($\uparrow$) & WGA($\uparrow$) & HM($\uparrow$)

        \\
        \midrule
        \midrule
        \multirow{6}{*}{\centering ZS}
        & CLIP-RN-50     
          & 88.7 & 41.0 & 56.1
          & 81.6 & 75.2 & 78.3
          & 91.8 & 63.3 & 74.9
          & 75.5 & 34.1 & 47.0 \\
        & CLIP-ViT-B/32  
          & 80.4 & 27.5 & 41.0
          & 78.3 & 68.9 & 73.3
          & 96.6 & 82.1& 88.8
          & 75.4 & 20.5& 32.2 \\
        & CLIP-ViT-L/14  
          & 88.6 & 27.6 & 42.1
          & 80.5 & 74.0 & 77.1
          & 98.1 & 79.8& 88.0
          & 72.4 & 4.1 & 7.8 \\
        & ALIGN            
          & 72.3 & 50.0 & 59.1
          & 82.4 & 78.2 & 80.2
          & 95.8 & 69.6 & 80.6
          & 78.5 & 34.1 & 47.5 \\
        & AltCLIP          
          & 90.3 & 37.2 & 52.7
          & 82.9 & 80.2 & 81.5
          & 98.5 & 82.5 & 89.8
          & 78.8 & 22.0 & 34.4 \\
        \cmidrule(lr){2-14}
        & \cellcolor[HTML]{FFFFFF}\textbf{Average}
          & \cellcolor[HTML]{EFEFEF}84.1
          & \cellcolor[HTML]{EFEFEF}36.7
          & \cellcolor[HTML]{EFEFEF}51.1
          & \cellcolor[HTML]{EFEFEF}81.1
          & \cellcolor[HTML]{EFEFEF}75.3
          & \cellcolor[HTML]{EFEFEF}78.1
          & \cellcolor[HTML]{EFEFEF}96.2
          & \cellcolor[HTML]{EFEFEF}75.5
          & \cellcolor[HTML]{EFEFEF}84.6
          & \cellcolor[HTML]{EFEFEF}76.1
          & \cellcolor[HTML]{EFEFEF}23.0
          & \cellcolor[HTML]{EFEFEF}35.3
          \\
        \midrule

        \multirow{6}{*}{\centering ROBOSHOT}
        & CLIP-RN-50     
          & 72.1 & 27.6 & 39.9
          & 81.6 & 74.9 & 78.1
          & 92.3 & 72.4 & 81.1
          & 77.6 & 37.6 & 50.7 \\
        & CLIP-ViT-B/32  
          & 74.2 & 39.3 & 51.4
          & 82.1 & 75.2 & 78.5
          & 96.6 & 83.5& 89.6
          & 77.1 & 35.2& 48.3 \\
        & CLIP-ViT-L/14  
          & 79.8 & 48.1 & 60.0
          & 85.3 & 82.2& 83.7
          & 98.0 & 81.3 & 88.9
          & 70.9 & 12.2 & 20.8\\
        & ALIGN            
          & 52.6 & 38.3 & 44.3
          & 87.0 & 84.8& 85.9
          & 94.7 & 63.2 & 75.8
          & 77.4  & 39.8 & 52.6 \\
        & AltCLIP          
          & 78.5 & 54.2 & 64.1
          & 86.1 & 80.6 & 83.3
          & 98.8 & 89.4 & 93.9
          & 78.3 & 25.7 & 38.7\\
        \cmidrule(lr){2-14}
        & \cellcolor[HTML]{FFFFFF}\textbf{Average}
          & \cellcolor[HTML]{EFEFEF}71.4
          & \cellcolor[HTML]{EFEFEF}\underline{41.5}
          & \cellcolor[HTML]{EFEFEF}52.5
          & \cellcolor[HTML]{EFEFEF}84.4
          & \cellcolor[HTML]{EFEFEF}\underline{79.5}
          & \cellcolor[HTML]{EFEFEF}\underline{81.9}
          & \cellcolor[HTML]{EFEFEF}96.1
          & \cellcolor[HTML]{EFEFEF}\underline{78.0}
          & \cellcolor[HTML]{EFEFEF}\underline{86.1}
          & \cellcolor[HTML]{EFEFEF}76.3
          & \cellcolor[HTML]{EFEFEF}30.1
          & \cellcolor[HTML]{EFEFEF}43.2
          \\
        \midrule

          \multirow{6}{*}{\centering TIE*}
        & CLIP-RN-50     
          & 83.8 & 34.1 & 48.5
          & 72.2 & 65.8 & 68.9
          & 88.6 & 54.7 & 67.6
          & 79.3 & 40.5 & 53.6 \\
        & CLIP-ViT-B/32  
          & 86.3 & 55.8 & 67.8
          & 85.9 & 69.1 & 76.6
          & 97.3 & 83.1 & 89.6
          & 79.7 & 32.4 & 46.1 \\
        & CLIP-ViT-L/14  
          & 87.6 & 39.1 & 54.1
          & 88.2 & 83.5 & 85.8
          & 97.7 & 82.5 & 89.5
          & 80.3 & 34.1 & 47.9\\
        & ALIGN            
          & 80.9 & 42.2 & 55.5
          & 86.6 & 82.2 & 84.3
          & 96.4 & 76.4 & 85.2
          & 81.3  & 26.2 & 39.6 \\
        & AltCLIP          
          & 82.9 & 21.0 & 33.5
          & 50.6 & 48.6& 49.6
          & 98.6 & 89.2& 93.7
          & 81.9 & 25.3 & 38.7\\
        \cmidrule(lr){2-14}
        & \cellcolor[HTML]{FFFFFF}\textbf{Average}
          & \cellcolor[HTML]{EFEFEF}84.3
          & \cellcolor[HTML]{EFEFEF}38.4
          & \cellcolor[HTML]{EFEFEF}\underline{52.8}
          & \cellcolor[HTML]{EFEFEF}76.7
          & \cellcolor[HTML]{EFEFEF}69.8
          & \cellcolor[HTML]{EFEFEF}73.1
          & \cellcolor[HTML]{EFEFEF}95.7
          & \cellcolor[HTML]{EFEFEF}77.2
          & \cellcolor[HTML]{EFEFEF}85.5
          & \cellcolor[HTML]{EFEFEF}80.5
          & \cellcolor[HTML]{EFEFEF}\underline{31.7}
          & \cellcolor[HTML]{EFEFEF}\underline{45.5}
          \\
          \midrule

                  \multirow{6}{*}{\centering \textbf{Ours (SAGE)}}
        & CLIP-RN-50     
          & 91.5 & 41.3 & 56.9
          & 82.2 & 77.5 & 79.8
          & 91.9 & 63.6 & 75.2
          & 74.8 & 36.8 & 49.3 \\
        & CLIP-ViT-B/32  
          & 92.3 & 46.0 & 61.4
          & 79.6 & 76.0 & 77.8
          & 97.0 & 85.0 & 90.6
          & 76.9 & 38.1 & 51.0 \\
        & CLIP-ViT-L/14  
          & 90.2 & 47.8 & 62.5
          & 85.7 & 83.9 & 84.8
          & 98.3 & 84.6 & 90.9
          & 74.6 & 23.9 & 36.2\\
        & ALIGN            
          & 81.6 & 47.0 & 59.6
          & 84.3 & 82.3 & 83.3
          & 97.6 & 87.0 & 92.0
          & 72.9  & 37.5 & 49.5 \\
        & AltCLIP          
          & 89.1 & 42.6 & 57.6
          & 85.2 & 83.3 & 84.2
          & 98.4 & 89.5 & 93.7
          & 79.9 & 32.5 & 46.2\\
        \cmidrule(lr){2-14}
        & \cellcolor[HTML]{FFFFFF}\textbf{Average}
          & \cellcolor[HTML]{EFEFEF}88.9
          & \cellcolor[HTML]{EFEFEF}\textbf{44.9}
          & \cellcolor[HTML]{EFEFEF}\textbf{59.7}
          & \cellcolor[HTML]{EFEFEF}83.4
          & \cellcolor[HTML]{EFEFEF}\textbf{80.6}
          & \cellcolor[HTML]{EFEFEF}\textbf{82.0}
          & \cellcolor[HTML]{EFEFEF}96.6
          & \cellcolor[HTML]{EFEFEF}\textbf{81.9}
          & \cellcolor[HTML]{EFEFEF}\textbf{88.7}
          & \cellcolor[HTML]{EFEFEF}75.8
          & \cellcolor[HTML]{EFEFEF}\textbf{33.8}
          & \cellcolor[HTML]{EFEFEF}\textbf{46.7}
          \\
        \bottomrule
    \end{tabular}
    }
    \caption{Performance on fine-grained (Waterbirds, CelebA) and coarse-grained (PACS, VLCS) spurious correlation benchmarks. The best worst-group accuracy (WGA) and harmonic mean (HM) are shown in \textbf{boldface}, while the second best WGA and HM are shown in \underline{underline}. Unlike ROBOSHOT and TIE*, which assume prior knowledge of spurious attributes, our method requires no such information.}
    \label{tab:results}
\end{table*}

\subsection{Spuriousness-Aware Guided Exploration}
\paragraph{Prompt-based class separation.}
Given a fixed class label $c_i$ and an input image embedding $\mathbf{v}$, we construct a set of $M$ prompt templates $\mathcal{T} = \{T_j\}_{j=1}^M$ and generate class-specific textual descriptions $\mathcal{D}_i = \{T_j(c_i)\}_{j=1}^M$, where $T_j(c_i)$ denotes the $j$-th prompt filled with class $c_i$. For example,  as illustrated in Figure~\ref{fig:method-overview}(b) with $C=2$ and $M=3$, when $c_1$ is ``landbird", $T_1(c_1)$ could be ``a photo of a landbird". These prompts are encoded via the text encoder $\psi$, producing corresponding representations $\mathbf\{u^j_i\}_{j=1}^M$, where $\mathbf{u}^j_i = \psi(T_j(c_i))$. Given $N$ test images denoted as $\{x_n\}_{n=1}^N$, we obtain their visual embeddings using the vision encoder $\phi$, resulting in $\mathbf{v}_n = \phi(x_n)$ for each $n = 1, 2, \dots, N$.

We evaluate each prompt template based on its ability to separate different classes in the joint image-text embedding space. Concretely, for a given prompt template $T_j$ and image $x_n$, we compute its cosine similarity with the image embedding $\mathbf{v_n}$ across all classes $c_1,\ldots,c_C$. We then define the \textit{separation score} of $T_j$ for $x_n$ as:
\begin{equation}
\sigma_j^n = \max_{i\in \{1,\cdots, C\}} \frac{\mathbf{v}_n^T \mathbf{u}^j_i}{\|\mathbf{v}_n\|_2 \|\mathbf{u}^j_i\|_2} - \min_{i \in \{1,\cdots, C\}} \frac{\mathbf{v}_n^T \mathbf{u}^j_i}{\|\mathbf{v}_n\|_2 \|\mathbf{u}^j_i\|_2}.\label{equ:score}
\end{equation}
A higher $\sigma_j^n$ indicates that the prompt better distinguishes between classes in terms of alignment with the image embedding, suggesting it is less biased and more informative.

\paragraph{Template selection and zero-shot inference.}
For each image, we rank all prompt templates based on their \textit{separation scores} $\sigma_j^n$ and select the top-$K$ templates with the highest scores. Here, $K$ is a hyperparameter that determines how many top-scoring templates are chosen for zero-shot inference. These templates are then used to construct $K$ zero-shot classifiers. As illustrated in Figure~\ref{fig:method-overview}(c), we use $K=2$ as an example, where the top two prompt templates are selected for zero-shot inference. For each selected template $T_k$, we compute the text embeddings $\mathbf{u}_i^{k} = \psi(T_k(c_i))$ for all class labels $i=1,\ldots,C$. The final prediction for the $n$-th image is obtained by averaging the similarity scores across the $K$ classifiers:
\begin{equation}
    \hat{y_n} = \arg\max_{i}\frac{1}{K} \sum_{k=1}^{K}\frac{\mathbf{v}_n^T\mathbf{u}_i^{k}}{\|\mathbf{v}_n\|_2\|\mathbf{u}_i^{k}\|_2}.
\end{equation}
This procedure enables a robust ensemble of diverse, class-discriminative templates without requiring any external knowledge of spurious attributes.

\begin{table*}[t]
    \centering
    \footnotesize
    \setlength{\tabcolsep}{2pt}
    \begin{tabular}{llccc|ccc|ccc}
        \toprule
        \multirow{3}{*}{\centering Dataset} 
        & \multirow{3}{*}{\centering Model}
        & \multicolumn{3}{c}{\multirow{1}{*}{\centering Ensemble (K=80)}}
        & \multicolumn{3}{c}{\multirow{1}{*}{\centering Random}}
        & \multicolumn{3}{c}{\multirow{1}{*}{\centering Ours (SAGE)}}
        \\
        \cmidrule(lr){3-5}\cmidrule(lr){6-8}\cmidrule(lr){9-11}
        & 
        & AVG($\uparrow$) & WGA($\uparrow$) & HM($\uparrow$)
        & AVG($\uparrow$) & WGA($\uparrow$) & HM($\uparrow$)
        & AVG($\uparrow$) & WGA($\uparrow$) & HM($\uparrow$)

        \\
        \midrule \midrule
        \multirow{6}{*}{\centering Waterbirds}
        & CLIP-RN-50     
          & 92.7 & 48.4 & 63.6
          & 87.3 & 45.8 & 60.1
          & 91.5 & 41.3 & 56.9\\
        & CLIP-ViT-B/32  
          & 92.7 & 23.5 & 37.5
          & 91.4 & 34.2 & 49.8
          & 92.3 & 46.0 & 61.4\\
        & CLIP-ViT-L/14  
          & 93.2 & 33.3 & 49.1
          & 92.1 & 39.6 & 55.4
          & 90.2 & 47.8 &62.5\\
        & ALIGN            
          & 81.7 & 46.1 & 58.9
          & 80.0 & 46.6 & 58.9
          & 81.6 & 47.0 & 59.6\\
        & AltCLIP          
          & 88.3 & 29.6 & 44.3
          & 88.0& 34.1 & 49.2
          & 89.1 & 42.6 &57.6\\
        \cmidrule(lr){2-11}
        & \cellcolor[HTML]{FFFFFF}\textbf{Average}
          & \cellcolor[HTML]{EFEFEF}89.7
          & \cellcolor[HTML]{EFEFEF}36.2
          & \cellcolor[HTML]{EFEFEF}51.6
          & \cellcolor[HTML]{EFEFEF}87.8
          & \cellcolor[HTML]{EFEFEF}40.1
          & \cellcolor[HTML]{EFEFEF}55.1
          & \cellcolor[HTML]{EFEFEF}88.9
          & \cellcolor[HTML]{EFEFEF}\textbf{44.9}
          & \cellcolor[HTML]{EFEFEF}\textbf{59.7}
          \\
        \midrule

        \multirow{6}{*}{\centering CelebA}
        & CLIP-RN-50     
          & 79.1 & 70.5 & 74.6
          & 75.3 & 68.1 & 71.5
          & 82.2 & 77.5 & 79.8\\
        & CLIP-ViT-B/32  
          & 77.8 & 67.6 & 72.3
          & 76.0 & 69.1& 72.4
          & 79.6 & 76.0 &77.8\\
        & CLIP-ViT-L/14  
          & 82.0 & 75.5 & 78.6
          & 82.8 & 80.4 & 81.6
          & 85.7 & 83.9 &84.8\\
        & ALIGN            
          & 80.2 & 74.4 & 77.2
          & 81.0  & 76.7 & 78.8
          & 84.3 & 82.3 &83.3\\
        & AltCLIP          
          & 81.8 & 78.1 & 79.9
          & 83.7 & 80.0 & 81.8
          & 85.2 & 83.3&84.2\\
        \cmidrule(lr){2-11}
        & \cellcolor[HTML]{FFFFFF}\textbf{Average}
          & \cellcolor[HTML]{EFEFEF}80.2
          & \cellcolor[HTML]{EFEFEF}73.2
          & \cellcolor[HTML]{EFEFEF}76.5
          & \cellcolor[HTML]{EFEFEF}79.8
          & \cellcolor[HTML]{EFEFEF}74.9
          & \cellcolor[HTML]{EFEFEF}77.3
           & \cellcolor[HTML]{EFEFEF}83.4
          & \cellcolor[HTML]{EFEFEF}\textbf{80.6}
          & \cellcolor[HTML]{EFEFEF}\textbf{82.0}
          \\
         \bottomrule
    \end{tabular}
    \caption{Ablation results comparing our method (SAGE) with random prompt selection and prompt ensembling on the Waterbirds and CelebA datasets. \textbf{Bold} numbers indicate the best performance among the three. Our method consistently achieves higher worst-group accuracy (WGA) and harmonic mean (HM) of WGA and average accuracy (AVG), demonstrating its effectiveness in mitigating spurious correlations.}
    \label{tab:ablation_random}
\end{table*}

\section{Experiments}
\label{sec:experiments}
\subsection{Datasets}
We experiment on two datasets with \textit{fine-grained spurious correlations}, where each class is correlated with certain spurious features, such as backgrounds and gender.

\begin{itemize}
    \item \textbf{Waterbirds}~\cite{sagawa2019distributionally} is an image dataset for recognizing waterbirds and landbirds. It is generated synthetically by combining images of two kinds of birds from the CUB dataset~\cite{WelinderEtal2010} and the backgrounds, water and land,  from the Places dataset~\cite{zhou2017places}.
    \item \textbf{CelebA}~\cite{liu2015deep} is a large-scale image dataset of celebrity faces. The task is to identify hair color, non-blond or blond, with the gender as the spurious attributes.
\end{itemize}

\noindent We also experiment on two datasets with \textit{coarse-grained spurious correlations} where classes are associated with domain-specific features.

\begin{itemize}
    \item \textbf{PACS}~\cite{zhou2020deep} is a domain generalization dataset that includes four visually different styles: Photo, Art Painting, Cartoon, and Sketch. The task is to identify object categories (dog, elephant, giraffe, guitar, horse, house, person).
    \item \textbf{VLCS}~\cite{fang2013unbiased} is a domain generalization benchmark composed of four datasets: PASCAL VOC 2007~\cite{everingham2010pascal} (V), LabelMe~\cite{russell2008labelme} (L), Caltech101~\cite{bansal2023transfer} (C), and SUN09~\cite{choi2010exploiting} (S). It contains five overlapping classes (bird, car, chair, dog, and person) drawn from each dataset. The main challenge is to learn invariant features that generalize across these domains.
\end{itemize}

\subsection{Experimental Setup}
\noindent\textbf{Evaluated methods.} For zero-shot classification performance comparison, we evaluate the standard baseline ZS (Zero-Shot CLIP), as well as two recent state-of-the-art debiasing methods: ROBOSHOT \cite{adila2024zeroshot}, which utilizes large language models (LLMs) to identify spurious attributes and mitigate multimodal spurious bias, and TIE* \cite{ICLR2025_2649b1d4}, which introduces spurious prompts to infer pseudo-spurious labels and remove their influence from the image embeddings. 
Our proposed method, SAGE, by default selects the single prompt with the highest \textit{separation score}, i.e., setting $K = 1$ in Equation~\ref{equ:score}. All experiments were conducted on NVIDIA Quadro RTX 8000 GPUs (48GB).

\noindent\textbf{Models.} We evaluate CLIP and its variant models with different sizes and architectures: CLIP-RN-50, CLIP-ViT-B/32, CLIP-ViT-L/14, ALIGN~\cite{jia2021scaling}, and AltCLIP~\cite{chen2023altclip}. While the four ViT-based models are aligned with the setups in ROBOSHOT~\cite{adila2024zeroshot}, we additionally include CLIP-RN-50 to diversify the backbone architectures beyond Transformers. 


\noindent\textbf{Evaluation metrics.} We report the zero-shot classification performance of a model using three metrics: average accuracy (AVG), worst-group accuracy (WGA), and the harmonic mean (HM) of AVG and HM.
WGA is the \textit{primary} robustness metric that measures the model’s worst-group performance in the test set \cite{sagawadistributionally}, which can reflect the overall robustness to spurious correlations under distribution shifts.
While AVG reflects overall performance, it can be dominated by the majority of the test set. To better demonstrate the overall performance for both prediction and robustness, we propose to use the harmonic mean (HM) of AVG and WGA as another main evaluation metric. The HM is defined as:
\begin{equation}
    \text{HM} = \frac{2 \cdot \text{AVG} \cdot \text{WGA}}{\text{AVG} + \text{WGA}},
\end{equation}
which is sensitive to low values and penalizes imbalanced performance.
In zero-shot multimodal classification, a model may achieve high AVG by performing well on most samples but still have low WGA on the worst group. HM emphasizes the importance of maintaining both high AVG and WGA. Therefore, a robust model should exhibit high WGA and HM to indicate good performance on both prediction and robustness \cite{hasna2004harmonic}.



\subsection{Main Results}
We evaluate the effectiveness of our method, SAGE, in mitigating multimodal spurious biases at both the \textit{fine-grained} level where each class is correlated with specific spurious features (e.g., Waterbirds and CelebA) and the \textit{coarse-grained} level where classes are associated with broader domain-specific features (e.g., PACS and VLCS). As shown in Table~\ref{tab:results}, SAGE consistently ranks among the best or second-best performers in both HM and WGA across different models and datasets, demonstrating strong overall effectiveness and robustness.

\begin{figure*}[t]
    \centering
    \includegraphics[width=\linewidth]{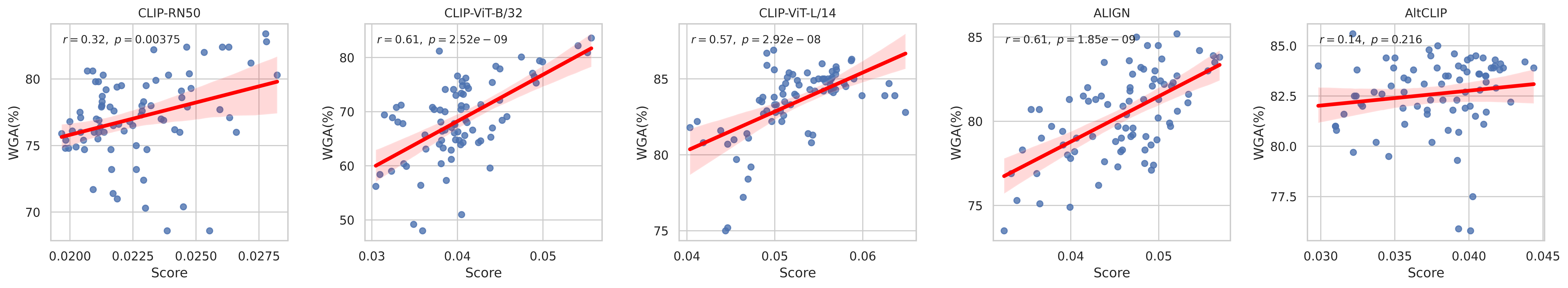}
    \caption{Pearson correlation analysis of \textit{separation scores} and WGA on CelebA across five backbone models. Each scatter plot shows the relationship between the score assigned to a candidate template and its corresponding WGA in zero-shot inference. The consistent positive correlation observed across all models indicates that templates with higher \textit{separation scores} tend to yield better worst-group performance, validating the effectiveness of our scoring method for robust template selection.}
    \label{fig:pcc_celebA}
\end{figure*}

\begin{figure*}[!htbp]
    \centering
    \includegraphics[width=\linewidth]{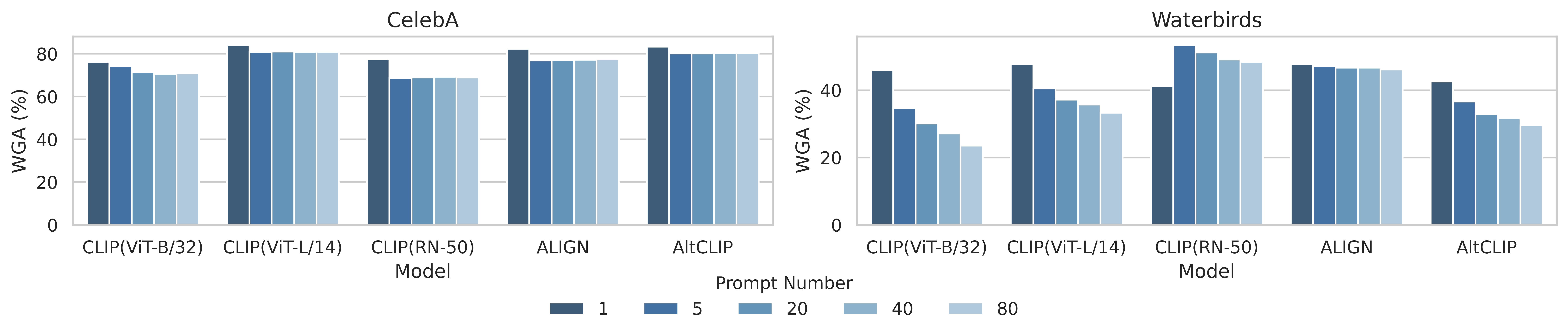}
    \caption{Ablation study on the effect of varying prompt numbers in different Models with our proposed method.}
    \label{fig:ablation}
\end{figure*}

To better capture performance consistency across varying model architectures, we report results averaged over five different backbones. This shows that SAGE consistently achieves the best WGA and HM, demonstrating strong balance between accuracy and fairness, and robust performance across subgroups.


These results highlight SAGE’s ability to reduce multimodal spurious bias without sacrificing predictive accuracy. Unlike ROBOSHOT and TIE*, which assume prior knowledge of spurious attributes, SAGE uses a fixed and general set of 80 prompts for controlled experiments, making it a practical, out-of-the-box debiasing solution for CLIP-style models. Note that SAGE is not limited to this specific setting and can be generalized to other sets of prompt templates.

\subsection{Ablation Studies}
\paragraph{Correlation Analysis of Separation Score and Accuracy.} 
To better understand the effectiveness of our template selection strategy, we analyze the relationship between the \textit{separation score} (used to rank prompt templates) and the zero-shot classification performance. Specifically, we compute the Pearson correlation coefficient (PCC) between the \textit{separation score} and WGA of each candidate template.

We perform this analysis on the CelebA dataset, a challenging benchmark characterized by subtle attribute differences and imbalanced group distributions. These properties make it well-suited for evaluating the robustness of prompt selection across diverse model backbones. For each of the five model backbones, we plot the \textit{separation score} against the corresponding WGA and report the PCC. As shown in Figure~\ref{fig:pcc_celebA}, the results consistently show a positive correlation.

These findings validate that templates with higher \textit{separation scores} tend to yield better worst-group performance, demonstrating that the \textit{separation score} is a reliable indicator of prompt robustness.


\paragraph{Validating the Prompt Selection Strategy in SAGE.} 
We evaluate the effectiveness of using the \textit{separation score} to select prompt templates by comparing three strategies: using all prompts (Ensemble), randomly selecting one prompt (Random, averaged over 3 runs), and selecting the prompt with SAGE. As shown in Table~\ref{tab:ablation_random}, our score-based method consistently achieves the best performance across nearly all settings on both Waterbirds and CelebA, significantly improving worst-group accuracy (WGA) and overall accuracy. In contrast, ensemble and random strategies often include suboptimal prompts that degrade performance. These results confirm that the \textit{separation score} is a reliable and effective criterion for prompt selection in zero-shot inference.

\paragraph{Impact of Number of Selected Prompts on Performance.} 
Even though templates can be ranked by their scores, the optimal number of prompts to use at inference time is not obvious. To explore this, we conduct an ablation study by varying the number of top-ranked templates $K$ used for zero-shot inference. We evaluate $K = 1, 5, 20, 40, 80$, ranging from using a single best prompt to all 80 templates.


Figure~\ref{fig:ablation} reports WGA across different $K$ values on Waterbirds and CelebA using five backbones. On CelebA (left), using only the top-1 prompt consistently achieves the best WGA, especially for CLIP-RN-50 and ALIGN. This aligns with CelebA’s fine-grained attributes, where precise prompt-image alignment is crucial and additional prompts may introduce noise.

On Waterbirds (right), the top-1 prompt yields the best performance for ViT-B/32, ViT-L/14, ALIGN, and AltCLIP, suggesting that prompts selected by SAGE remain effective even for large-scale pretrained models. In contrast, CLIP-RN-50 achieves its best results at $K=5$, indicating that moderate prompt diversity may help smaller models capture more robust and complementary visual cues.


Given these results, we choose $K=1$ as the default setting for SAGE, as it consistently provides strong performance across datasets and model sizes. While SAGE selects prompt templates for each image, we also observe interesting trends in the most frequently selected templates. 



\section{Conclusion}

In this paper, we addressed the challenge of mitigating multimodal spurious biases in pre-trained CLIP models for zero-shot classification. We first provided a theoretical definition of multimodal spurious bias and analyzed its impact on zero-shot classification. Based on these insights, we introduced SAGE, which is inspired by our theoretical insights. Our approach operates out-of-the-box with CLIP models, requiring no additional training data or prior knowledge of biases. It is broadly effective across various model sizes, architectures, and types of spurious correlations. Moreover, it achieves a strong balance between average and worst-group zero-shot classification accuracy, highlighting its practical utility in robust zero-shot predictions.

\section*{Acknowledgements}
This work is supported in part by the US National Science Foundation under grants CCF-2217071, CNS-2213700, IIS-2106913.  Any opinions, findings, and conclusions or recommendations expressed in this material are those of the
author(s) and do not necessarily reflect the views of the National Science Foundation.

\bibliography{aaai2026}

\appendix

\section*{Appendix}

\subsection*{Prompt Templates}

We provide the prompt templates used in the experiments in Table~\ref{tab:prompt_templates}. There are a total of 80 templates. The special symbol ``[CLASS]" is a placeholder, which will be replaced with actual class labels in zero-shot classification.

For the vanilla zero-shot classification method, we followed the prompts used in \cite{adila2024zeroshot}. Specifically, on the Waterbirds dataset, we used ``an image of landbird" and ``an image of waterbird"; on the CelebA dataset, we used ``person with dark hair" and ``person with blond hair"; on the PACS and VLCS datasets, we directly used the class names as the input text descriptions.
\begin{table*}[]
    \centering
    \begin{tabular}{llccl}
        \toprule
        \textbf{Dataset} 
         & \textbf{Groups}
         & \multicolumn{2}{c}{\textbf{Statistics}}
         & \textbf{Classes}\\
        \cmidrule(lr){3-4}
         & 
         & \textbf{Total Samples}
         & \textbf{\# Classes}
         & \\ 
        \midrule
        Waterbirds
        & \makecell[l]{landbird in land, landbird in water,\\
                       waterbird on land, waterbird on water}
        & 5794
        & 2
        & landbird, waterbird \\
        \midrule
        CelebA
        & \makecell[l]{male \& not blond, female \& not blond,\\
                       male \& blond, female \& blond}
        & 19962
        & 2
        & not blond, blond \\
        \midrule
        PACS
        & \makecell[l]{art, cartoons, photos,\\
                       sketches}
        & 9991
        & 7
        & \makecell[l]{dogs, elephant, giraffe,\\
                       guitar, house, person} \\
        \midrule
        VLCS
        & \makecell[l]{Caltech101, LabelMe,\\
                       SUN09, VOC2007}
        & 10725
        & 5
        & bird, car, chair, dog, person \\
        \bottomrule
    \end{tabular}
    \caption{Dataset statistics including groups, total samples, number of classes, and class labels.}
    \label{tab:dataset_statistics}
\end{table*}

\begin{table*}[]
    \centering
    \begin{tabular}{|l|l|}
        \hline
        \textbf{Prompt Templates} & \textbf{Prompt Templates} \\
        \hline
        a bad photo of a [CLASS]. & a photo of many [CLASS]. \\
        a sculpture of a [CLASS]. & a photo of the hard to see [CLASS]. \\
        a low resolution photo of the [CLASS]. & a rendering of a [CLASS]. \\
        graffiti of a [CLASS]. & a bad photo of the [CLASS]. \\
        a cropped photo of the [CLASS]. & a tattoo of a [CLASS]. \\
        the embroidered [CLASS]. & a photo of a hard to see [CLASS]. \\
        a bright photo of a [CLASS]. & a photo of a clean [CLASS]. \\
        a photo of a dirty [CLASS]. & a dark photo of the [CLASS]. \\
        a drawing of a [CLASS]. & a photo of my [CLASS]. \\
        the plastic [CLASS]. & a photo of the cool [CLASS]. \\
        a close-up photo of a [CLASS]. & a black and white photo of the [CLASS]. \\
        a painting of the [CLASS]. & a painting of a [CLASS]. \\
        a pixelated photo of the [CLASS]. & a sculpture of the [CLASS]. \\
        a bright photo of the [CLASS]. & a cropped photo of a [CLASS]. \\
        a plastic [CLASS]. & a photo of the dirty [CLASS]. \\
        a jpeg corrupted photo of a [CLASS]. & a blurry photo of the [CLASS]. \\
        a photo of the [CLASS]. & a good photo of the [CLASS]. \\
        a rendering of the [CLASS]. & a [CLASS] in a video game. \\
        a photo of one [CLASS]. & a doodle of a [CLASS]. \\
        a close-up photo of the [CLASS]. & a photo of a [CLASS]. \\
        the origami [CLASS]. & the [CLASS] in a video game. \\
        a sketch of a [CLASS]. & a doodle of the [CLASS]. \\
        an origami [CLASS]. & a low resolution photo of a [CLASS]. \\
        the toy [CLASS]. & a rendition of the [CLASS]. \\
        a photo of the clean [CLASS]. & a photo of a large [CLASS]. \\
        a rendition of a [CLASS]. & a photo of a nice [CLASS]. \\
        a photo of a weird [CLASS]. & a blurry photo of a [CLASS]. \\
        a cartoon [CLASS]. & art of a [CLASS]. \\
        a sketch of the [CLASS]. & an embroidered [CLASS]. \\
        a pixelated photo of a [CLASS]. & itap of the [CLASS]. \\
        a jpeg corrupted photo of the [CLASS]. & a good photo of a [CLASS]. \\
        a plushie [CLASS]. & a photo of the nice [CLASS]. \\
        a photo of the small [CLASS]. & a photo of the weird [CLASS]. \\
        the cartoon [CLASS]. & art of the [CLASS]. \\
        a drawing of the [CLASS]. & a photo of the large [CLASS]. \\
        a black and white photo of a [CLASS]. & the plushie [CLASS]. \\
        a dark photo of a [CLASS]. & itap of a [CLASS]. \\
        graffiti of the [CLASS]. & a toy [CLASS]. \\
        itap of my [CLASS]. & a photo of a cool [CLASS]. \\
        a photo of a small [CLASS]. & a tattoo of the [CLASS]. \\
        \hline
    \end{tabular}
    \caption{List of prompt templates.}
    \label{tab:prompt_templates}
\end{table*}

\subsection*{Dataset Details}
The details of the four datasets used in the experiments are shown in Table~\ref{tab:dataset_statistics}, including groups, total samples, number of classes, and class labels. As we focus on the zero-shot setting, only the information regarding the test set in each dataset is shown in Table~\ref{tab:dataset_statistics}.

\subsection*{Computing infrastructure.}
All experiments were conducted on a single NVIDIA Quadro RTX 8000 GPU (48GB) with 251GB RAM, using PyTorch 2.6 and the OpenCLIP implementation of CLIP. The operating system was Ubuntu 22.04. No training was performed; all results are from zero-shot inference using pre-trained CLIP models.

\subsection*{Limitations and Future Works}
While our proposed method demonstrates significant robustness, the performance of SAGE is contingent on the diversity and quality of the predefined prompt templates. A more diverse and task-relevant set of prompt templates could enhance the method's ability to select optimal prompts for mitigating multimodal spurious biases. Furthermore, SAGE operates within the framework of zero-shot debiasing, meaning it does not incorporate any training techniques for vision-language models (VLMs). Although this ensures the approach remains entirely out-of-the-box, future work could explore integrating SAGE with small labeled datasets to further refine and improve model performance. Lastly, while we evaluated SAGE across multiple datasets, extending its evaluation to a broader range of tasks and bias types would provide deeper insights into its generalizability and broader applicability.

\begin{figure}[t]
    \centering
    \includegraphics[width=\linewidth]{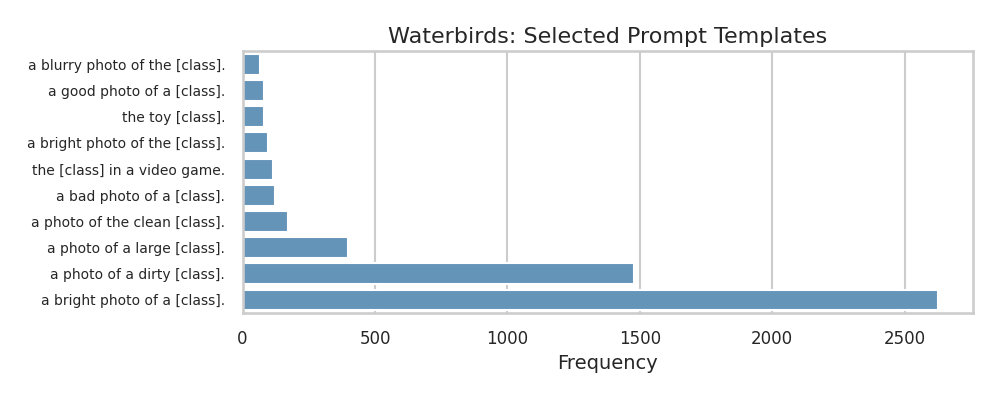}
    \vskip -10pt
    \caption{Most frequently selected prompt templates for each class by our method with CLIP-ViT-B/32 in the Waterbirds dataset.}
    \label{fig:top-prompts}
\end{figure}

\begin{figure}[t]
    \centering
    \includegraphics[width=\linewidth]{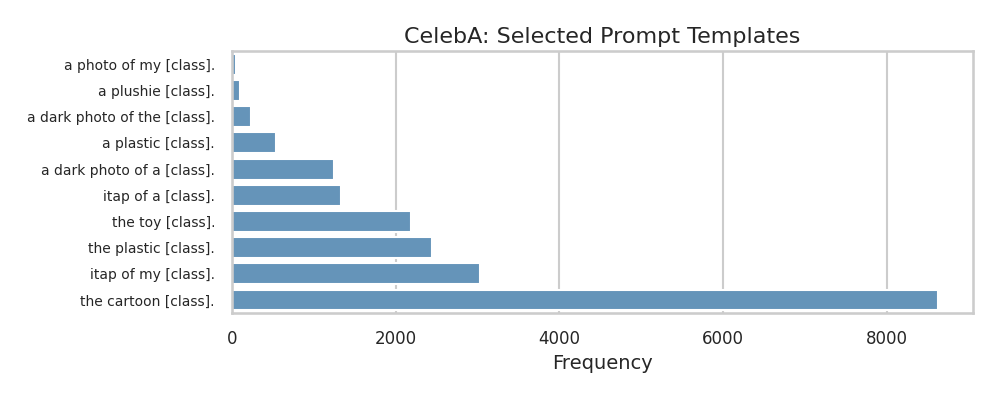}
    \vskip -10pt
    \caption{Top-10 most frequently selected prompt templates by our method for each class with CLIP-ViT-B/32 in the CelebA dataset.}
    \label{fig:top-prompts-celeba}
\end{figure}

\begin{figure}[t]
    \centering
    \includegraphics[width=\linewidth]{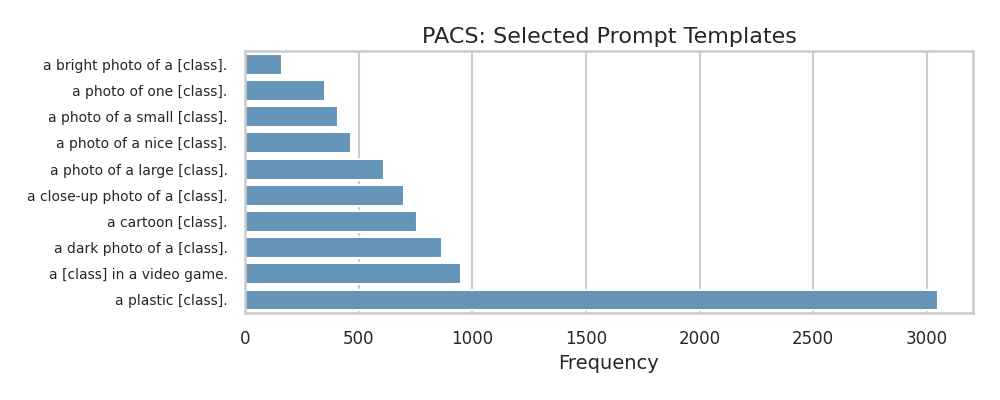}
    \vskip -10pt
    \caption{Top-10 most frequently selected prompt templates by our method  for each class with CLIP-ViT-B/32 in the PACS dataset.}
    \label{fig:top-prompts-pacs}
\end{figure}

\begin{figure}[t]
    \centering
    \includegraphics[width=\linewidth]{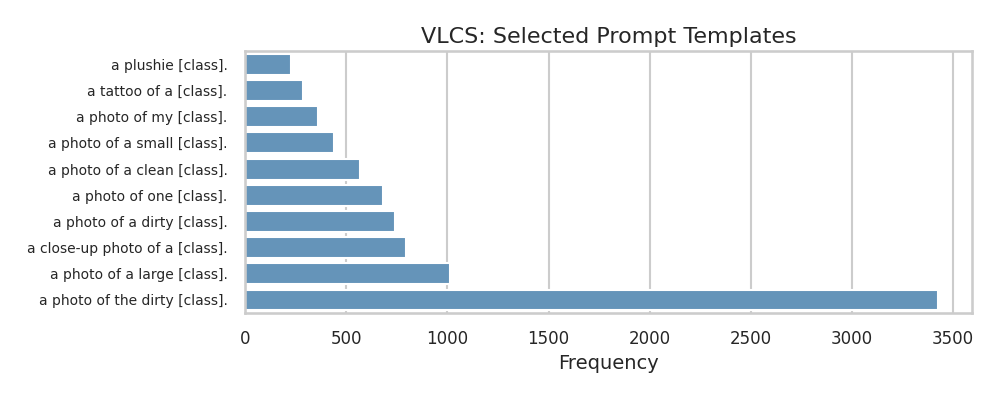}
    \vskip -10pt
    \caption{Top-10 most frequently selected prompt templates by our method  for each class with CLIP-ViT-B/32 in the VLCS dataset.}
    \label{fig:top-prompts-vlcs}
\end{figure}

\subsection*{Analysis on the Selected Prompt Templates}

\label{sec:anaylsis_selected}

Our method, SAGE, selects the highest-scoring prompt template for each test image. To better understand this selection behavior, we analyze which templates are most frequently chosen across the Waterbirds test set. Figure~\ref{fig:top-prompts} presents the top-10 most selected templates and their frequencies in the Waterbirds dataset.

The most frequently selected template is “a bright photo of a [CLASS]”. This template uses a neutral adjective, “bright”, which is generally unrelated to typical spurious features such as background. In the Waterbirds dataset, where background often confounds classification, prompts like this may help the model focus more on object-relevant features rather than contextual features. The second most common template is “a photo of a dirty [CLASS]”, which includes an uncommon description for the classes in the dataset. This unusual wording might cause the text embedding to shift away from the typical distribution seen during training, potentially reducing spurious correlations between text and image.

Our method applies the selected prompt uniformly across all classes. This ensures that the model’s predictions are primarily influenced by the visual input rather than differences in prompt wording, which helps maintain consistency and avoids introducing additional variability.

We show more prompt templates selected by our method in Figures ~\ref{fig:top-prompts-celeba}, \ref{fig:top-prompts-pacs}, and \ref{fig:top-prompts-vlcs}. We observe that, in general, the most frequently selected template is different across classes and datasets. Interestingly, we observe that one specific prompt template is overwhelmingly favored across all test images within each dataset. This suggests that certain templates inherently provide stronger class separation in the embedding space, possibly due to their neutral semantics, alignment with pretraining distribution, or globally optimal positioning in the joint space. Such behavior highlights the potential of our scoring-based selection strategy to identify robust and broadly effective prompts without requiring dataset-specific tuning.





\end{document}